\documentclass[letterpaper, 10 pt, conference]{ieeeconf}  

\IEEEoverridecommandlockouts                              

\newif\ifanonymous
\anonymousfalse   

\usepackage{moreverb,url}
\usepackage{xspace}
\usepackage{xcolor}
\usepackage{siunitx}

\usepackage{amssymb}
\usepackage{amsmath}
\usepackage{amsbsy}
\usepackage{bbm}
\usepackage{dsfont}

\usepackage{amsthm}

\usepackage{cite}
\usepackage{amsfonts}
\usepackage{graphicx}
\usepackage{textcomp}
\usepackage{xcolor}
\usepackage{color}
\usepackage{tabularray}
\usepackage{tabularx}

\usepackage{hyperref}

\usepackage{graphicx}
\usepackage{wrapfig}
\usepackage{booktabs} 

\usepackage{subfiles}
\usepackage{csquotes}

\usepackage{array}

\usepackage{graphicx}
\usepackage{subfig} 
\usepackage{multirow}
\newcommand\BibTeX{{\rmfamily B\kern-.05em \textsc{i\kern-.025em b}\kern-.08em
T\kern-.1667em\lower.7ex\hbox{E}\kern-.125emX}}

\definecolor{wine}{RGB}{204, 0, 102}
\definecolor{magenta_wine}{RGB}{158, 44, 143}
\definecolor{dusty_wine}{RGB}{143, 59, 101}
\definecolor{ocean}{RGB}{13, 121, 202}
\definecolor{light_ocean}{RGB}{18, 178, 235}
\definecolor{dark_ocean}{RGB}{10, 89, 148}
\definecolor{grey}{RGB}{170, 170, 170}
\definecolor{light-grey}{RGB}{220, 220, 220}
\definecolor{dark_gray}{rgb}{0.2, 0.2, 0.2} 
\definecolor{med-grey}{rgb}{0.3, 0.3, 0.3} 
\definecolor{grape}{RGB}{112,48,160}
\definecolor{aqua}{RGB}{52,172,139}
\definecolor{dark_aqua}{RGB}{35,115,93}
\definecolor{dark_orange}{RGB}{216,92,0}
\definecolor{vibrant_orange}{RGB}{250, 160, 26}
\definecolor{vibrant_blue}{RGB}{14, 120, 255}
\definecolor{vibrant_pink}{RGB}{255, 0, 104}
\definecolor{dark_red}{RGB}{122, 0, 0}
\definecolor{dark_green}{RGB}{0, 92, 34}
\definecolor{dusty_blue}{RGB}{77, 91, 128}
\definecolor{dark_brown}{RGB}{125, 54, 36}

\newcommand{\para}[1]{\medskip\noindent\textbf{#1. }} 
 
\newcommand{\paranopunc}[1]{\medskip\noindent\textbf{#1 }} 

\newcounter{qnum}
\newcommand{\latentsafe}{\textcolor{dark_ocean}{\textbf{LatentSafe}}\xspace}
\newcommand{\ours}{\textcolor{vibrant_orange}{\textbf{AnySafe}}\xspace}  %

\newtheorem{theorem}{Theorem}

\usepackage{pifont}
\newcommand{\cmark}{\ding{51}} 
\newcommand{\xmark}{\ding{55}}%

\newcommand{\state}{s}

\newcommand{\latent}{z}
\newcommand{\latentSpace}{\mathcal{Z}}
\newcommand{\latentConstraint}{\latent_c}
\newcommand{\dynz}{f_{\latent}}
\newcommand{\dynamics}{f}
\newcommand{\encoder}{\mathcal{E}}

\newcommand{\obs}{o}
\newcommand{\obsSpace}{\mathcal{O}}

\newcommand{\action}{a}
\newcommand{\actionSpace}{\mathcal{A}}

\newcommand{\policy}{\pi}

\newcommand{\policyTask}{\pi^{\text{task}}}

\usepackage{fontawesome5}
\newcommand{\shield}{\text{\tiny{\faShield*}}}

\newcommand{\failure}{\mathcal{F}}

\newcommand{\ellz}{\ell_\latent}

\newcommand{\latentfailuremargin}{l}

\newcommand{\valfunc}{V}

\newcommand{\monitor}{{\valfunc^{\shield}}} 
\newcommand{\monitorDelta}{\valfunc^{\shield}_\delta}
\newcommand{\fallback}{\policy^{\shield}} 

\newcommand{\dataset}{\mathcal{D}}
\newcommand{\train}{\text{train}}
\newcommand{\calib}{\text{calib}}
\newcommand{\test}{\text{test}}

\newcommand{\semanticEncoder}{\tilde{\encoder}}
\newcommand{\latentSemanticSpace}{\tilde{\mathcal{Z}}}
\newcommand{\latentSemantic}{\tilde{\latent}}
\newcommand{\latentSemanticConstraint}{\tilde{\latent}_c}
\newcommand{\similarity}{\operatorname{sim}}
\newcommand{\threshold}{\delta}
\newcommand{\gtclass}{y}
\newcommand{\prototype}{p}
\newcommand{\alphaCalibration}{\alpha}

\newcommand{\latentConstraintTraining}{\latent_i}
\newcommand{\latentSemanticConstraintTraining}{\tilde{\latent}_i}

\newcommand{\anysafe}{{\textit{AnySafe}}\xspace}

\title{\LARGE \bf
AnySafe: Adapting Latent Safety Filters at Runtime via \\ Safety Constraint Parameterization in the Latent Space
}

\ifanonymous
  \author{Anonymous Authors}
\else
    \author{Sankalp Agrawal$^{*1}$, Junwon Seo$^{*2}$, Kensuke Nakamura$^2$, Ran Tian$^{3,4}$, Andrea Bajcsy$^2$ \thanks{$^{*}$These authors contributed equally to this work. $^{1}$The Ohio State University. {\tt\footnotesize agrawal.268@buckeyemail.osu.edu}. $^{2}$Carnegie Mellon University. {\tt\footnotesize \{junwonse, kensuken, abajcsy\}@andrew.cmu.edu}. $^{3}$UC Berkeley. {\tt\footnotesize rantian@berkeley.edu}. $^{4}$NVIDIA Research.}
    }
\fi

\begin{document}

\maketitle

\begin{abstract}

Recent works have shown that foundational safe control methods, such as Hamilton–Jacobi (HJ) reachability analysis, can be applied in the latent space of world models. While this enables the synthesis of latent safety filters for hard-to-model vision-based tasks, they assume that the safety constraint is known a priori and remains fixed during deployment, limiting the safety filter's adaptability across scenarios. To address this, we propose \textit{constraint-parameterized latent safety filters} that can adapt to user-specified safety constraints at runtime. Our key idea is to define safety constraints by conditioning on an encoding of an image that represents a constraint, using a latent-space similarity measure. The notion of similarity to failure is aligned in a principled way through conformal calibration, which controls how closely the system may approach the constraint representation. The parameterized safety filter is trained entirely within the world model's imagination, treating any image seen by the model as a potential test-time constraint, thereby enabling runtime adaptation to arbitrary safety constraints. In simulation and hardware experiments on vision-based control tasks with a Franka manipulator, we show that our method adapts at runtime by conditioning on the encoding of user-specified constraint images, without sacrificing performance. Video results can be found on the \href{https://any-safe.github.io/}{\textcolor{vibrant_orange}{project website}}.
\end{abstract}

\section{Introduction}
World models offer a promising paradigm for generalizing robot control to hard-to-simulate physical tasks by learning compact latent state spaces and dynamics directly from high-dimensional observations~\cite{hafner2019learning, hafner2023dreamerv3, hansen2024tdmpc2, zhou2024dino}. Recent works have demonstrated that foundational safe control methods, such as Hamilton–Jacobi (HJ) reachability analysis~\cite{mitchell2005time, wabersich2023data}, can be applied directly in a world model’s latent space, enabling safety analysis directly from high-dimensional sensor inputs. By computing robot policies that anticipate and avoid future failures within the world model’s imagination, these \textit{latent safety filters} can proactively steer robots away from hard-to-model constraints, such as spilling the contents of deformable bags~\cite{nakamura2025generalizing} or toppling complex rigid-body structures~\cite{seo2025uncertainty}.

However, most safe control frameworks assume that the state constraints that robots should avoid are determined \textit{a priori} and remain fixed during deployment~\cite{hsu2023safety, wabersich2023data}. In practice, this assumption is overly restrictive: at deployment time, a robot may need to adapt its notion of what is a safety constraint based on changing environments or end-user requirements. For example, consider the robot manipulator in Fig.~\ref{fig:main} that must sweep clutter from a table. In one scenario, it needs to avoid sweeping objects in a particular region (top row), but later it may be tasked with intentionally collecting objects into that same region while avoiding a different one (bottom row). This raises the central question of our work: 
\begin{quote}
\centering 
\textit{How can latent safety filters adapt to safety constraints specified at test-time?}
\end{quote}

\begin{figure}[t!]
\centering
\includegraphics[width=1.0\linewidth]{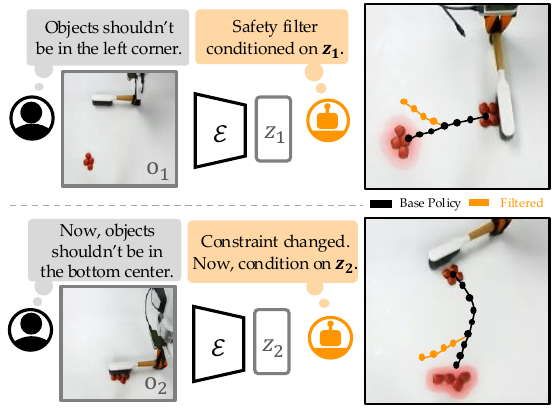}
\caption{\textbf{Constraint-Parameterized Latent Safety Filter.} The latent safety filter adapts its safety constraint by conditioning on an encoding of an image specified by the user as a failure.}
\label{fig:main}
\vspace{-0.2in}
\end{figure}

In this work, we design \textit{constraint-parameterized latent safety filters} (called \anysafe). The core challenge with parameterizing safety constraints in the latent space is that, unlike in hand-designed state spaces, the structure needed to represent and optimize against a suite of safety constraints does not naturally emerge. In hand-designed state spaces, one can design a low-dimensional parameterization of the constraint set (e.g., a circle by its center and radius) alongside a dense distance measure for guiding policy optimization (e.g., signed distance to the constraint set); this allows for the safety filter to be effectively computed for all possible constraint variations. In latent spaces, by contrast, constraints are typically only implicitly defined by classifiers on the latent states \cite{nakamura2025generalizing, seo2025uncertainty} which do not admit a continuous parameterization to represent diverse safety constraints nor yield a notion of proximity from a state to such constraints.

We propose three key ingredients that enable constraint-parameterization in latent safety filters. First, we specify safety constraints via a \textit{similarity measure} between the embedding of a constraint image and the robot's current latent state; this provides a dense signal of how close the policy is to failure. Then, we \textit{calibrate} the resulting constraint set with conformal prediction~\cite{shafer2008tutorial, angelopoulos2023conformal} to align with an end-user's semantic notion of failure. Lastly, we train the safety filter by treating \textit{any} image in the world model dataset as a possible test-time safety constraint. At runtime, we adapt the latent safety filter by conditioning it on an encoding of a user-specified constraint image, thereby adapting it to the runtime safety specification. 

We evaluate our framework on vision-based safe-control tasks, including a simulated vehicle collision-avoidance domain and real-world object sweeping with a Franka manipulator. Our results highlight four key findings: (1) by parameterizing the safety filter with constraint representations, \anysafe can adapt to arbitrary constraints provided as images; (2) this adaptability does not come at the cost of performance, as for a given constraint, the parameterized safety filter achieves performance comparable to a specialized filter trained solely on that constraint;
(3) \anysafe generalizes to constraints beyond those that specialized safety filters can model; and (4) since \anysafe learns from continuous latent similarity signals, conformal calibration allows us to control how conservatively the robot avoids specified constraints by adjusting the effective size of the failure set.

\begin{figure*}[t!]
\centering
\includegraphics[width=1.0\textwidth]{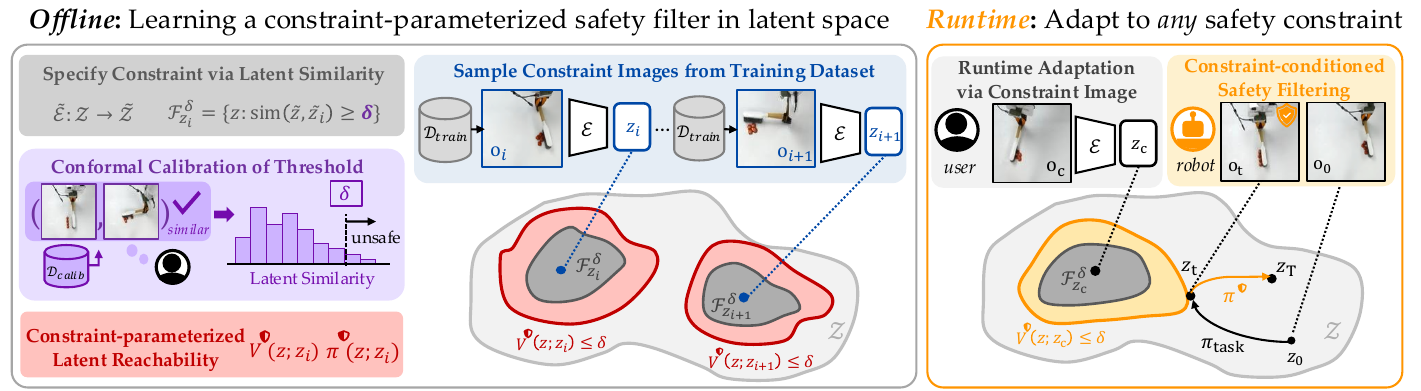}
\caption{\textbf{Framework: Constraint-Parameterized Latent Safety Filter.} \textit{Left}: The constraint-parameterized latent safety filter is trained by sampling constraint images from the WM training dataset, treating any image as a possible test-time safety constraint. Safety constraints are specified using a latent-space similarity measure, with a calibrated threshold that defines the size of the failure set. \textit{Right}: At runtime, the safety filter adapts to any safety constraint with a user-specified image.}
\label{fig:framework}
\vspace{-0.2in}
\end{figure*}
\section{Related Work}

\para{Safety Filtering in Robotics}
Safety filtering is a control-theoretic approach for preventing robotic systems from entering unsafe states~\cite{hsu2023safety, wabersich2023data}. Foundational methods enforce safety by correcting the robot's base control policy using Control Barrier Functions~\cite{ames2016control}, Hamilton–Jacobi (HJ) reachability analysis~\cite{mitchell2005time, margellos2011hamilton}, or model predictive shielding~\cite{bastani2021safe}. Since the key challenge of these methods is the tractable synthesis of a valid safety value (or barrier) function that encodes a set of safe states, recent methods have leveraged self-supervised learning~\cite{bansal2021deepreach} and reinforcement learning (RL)~\cite{fisac2019bridging, hsu2021safety} to scale safety value functions and controllers to high-dimensional nonlinear systems. More recently, latent dynamics models~\cite{hafner2023dreamerv3, zhou2024dino} have been used to compute these safety filters in learned latent spaces~\cite{nakamura2025generalizing, seo2025uncertainty}, enabling safe control directly from high-dimensional image observations.

However, most existing approaches assume that the safety constraint is fixed~\cite{bansal2017hamilton, wabersich2023data, hsu2023safety}, and thus cannot generalize beyond a predefined safety specification. With hand-designed state spaces and dynamics models, safety constraints can be updated online by precomputing families of reachable sets parameterized by environmental or system factors~\cite{borquez2023parameter}, or by incrementally updating safety specifications~\cite{bajcsy2019efficient, santos2025updating}. Recent works introduce observation-conditioned safety filters, which adapt a safety value function for collision-avoidance with current sensor observations~\cite{hsu2023sim, he2024agile, lin2024one}. In contrast, our method operates in the world model's learned latent state space and adapts the safety filter to runtime user-specified constraints, provided as RGB images of undesirable states. 

\para{Learning Conditioned Control Policies}
Conditioned (or parameterized) policies are a well-established way to enable control policies to adapt to arbitrary objectives.
Within the safe control literature, this conditioning has been applied to constraint thresholds~\cite{yao2023constraint}, low-dimensional environmental parameters~\cite{borquez2023parameter}, or high-dimensional LiDAR observations~\cite{he2024agile,lin2024one}, all of which adapt the robot's safety specifications at runtime. More broadly, goal-conditioned reinforcement learning (GCRL) trains agents to achieve diverse user-defined goals by conditioning policies and value functions on goal representations, enabling zero-shot deployment with runtime goal images without further training~\cite{kaelbling1993learning, schaul2015universal}. GCRL relabels past experiences in a self-supervised manner and learn value functions that capture similarities between states~\cite{eysenbach2022contrastive, mendonca2021discovering}. 
Instead of goal-conditioning, we \textit{constraint}-condition the HJ reachability problem with an RL-based solver~\cite{fisac2019bridging} in the imagination of a world model, self-labeling past trajectories as potential safety constraints to make a latent safety filter avoid diverse runtime safety constraints.

\section{Background: Latent Safety Filter}\label{sec:prelim}

We briefly introduce latent safety filters~\cite{nakamura2025generalizing}, which serve as the foundation of our method. A latent safety filter consists of two components: a safety value function \(\monitor: \latentSpace \rightarrow \mathbb{R}\), which measures how close the robot is to inevitable failures, and a safety-preserving policy \(\fallback: \latentSpace \rightarrow \actionSpace\), which steers the robot away from failure. 
In this work, we compute these models via Hamilton–Jacobi (HJ) reachability analysis~\cite{mitchell2005time}. 
The key innovation is that the value function and policy are optimized within the learned latent state representation ($\latent \in \latentSpace$) of a world model learned directly from RGB observations.

\para{Latent States, Dynamics, and Constraints} 
World models~\cite{hafner2019learning, hafner2023dreamerv3, zhou2024dino}
offer a paradigm for learning difficult-to-simulate dynamical systems models directly from raw sensor observations by jointly inferring a lower-dimensional latent state $\latent \in \latentSpace$ and its associated dynamics $\dynamics_\latent$. These models are trained with an offline dataset of robot–environment interactions, \(\dataset_\train := \{(\obs_t, \action_t, \latentfailuremargin_t)_{t=1}^T\}_{i=1}^{N_\train}\). Here, each trajectory consists of high-dimensional observations \(\obs \in \obsSpace\) (e.g., proprioception and RGB images), robot actions \(\action \in \actionSpace\), and failure labels \(\latentfailuremargin \in \{-1, 1\}\) that indicate the presence of visible failures from an observation. (e.g., the contents of a bag being spilled~\cite{nakamura2025generalizing}, or objects toppling into a sensitive region~\cite{seo2025uncertainty}). Note that the failure is modeled as a binary classification, where the notion of failure is assumed to be fixed at training time and remains unchanged at deployment.

The world model consists of an encoder $\encoder$ that maps an observation and a prior latent state $\hat{\latent} \in \latentSpace$ into the posterior latent representation $\latent \in \latentSpace$, and a latent dynamics model $\dynamics_z$ that predicts the next latent state conditioned on an action.\vspace{-0.05in}
\begin{equation}
\begin{aligned}
        \text{Encoder: } &\latent_t \sim \encoder(\latent_t \mid \hat{\latent}_t, \obs_t)  \\
        \text{Latent Dynamics: } &\hat{\latent}_t \sim \dynz(\hat{\latent}_t \mid \latent_{t-1}, \action_{t-1}) \\
        \text{Failure Classifier: } &\latentfailuremargin_t = \ellz(\latent_t), \label{eq:world_model}
\end{aligned} \vspace{-0.05in}
\end{equation}
where $\ellz(\latent_t)$ models the safety constraints as a classifier; it returns whether the $\latent$ is in failure or not. This formulation describes a wide range of world models~\cite{hafner2019learning, hafner2023dreamerv3, hansen2024tdmpc2, zhou2024dino, assran2025vjepa2}.

\para{Latent Safety Filter with a \textit{Fixed} Failure Set}
With a binary classifier $\ellz$ learned on the latent space, a fixed safety constraint can be represented as a \textit{failure set} $\failure := \{\latent : \ellz(z) \leq 0 \} \subset \latentSpace$ encoded via the zero-sublevel set of the failure margin function. A latent safety filter for the \textit{fixed} failure set can then be constructed by performing Hamilton–Jacobi (HJ) reachability analysis~\cite{mitchell2005time, hsu2023safety} in the latent space. Using the latent imagination of a pretrained world model as the environment dynamics, the safety filter is learned by solving the fixed-point safety Bellman equation~\footnote{This includes an expectation over transitions for stochastic dynamics (e.g., RSSM~\cite{hafner2019learning}) but can be removed for deterministic ones (e.g.,\cite{zhou2024dino}).}:
\begin{equation}\label{eqn:discounted-safety-bellman}
\begin{aligned}
    \monitor(\latent_t) &= (1-\gamma)\ellz(\latent_t) \\ & + \gamma \min \Big\{ \ellz(\latent_t), \max_{\action_t \in \actionSpace} \mathbb{E}_{\hat{\latent}_{t+1} \sim \dynz(\cdot \mid \latent_t, \action_t)}\monitor\big( \hat{\latent}_{t+1}\big) \Big\}, 
    \\
    \fallback(\latent_t) &= \arg\max_{\action \in \actionSpace} 
    \mathbb{E}_{\hat{\latent}_{t+1} \sim \dynz(\cdot \mid \latent_t, \action)} 
        \monitor\big(\hat{\latent}_{t+1}\big).
\end{aligned}
\end{equation} where $\gamma \in [0,1)$ is a time discounting factor ensuring a contraction mapping~\cite{fisac2019bridging}. 
We note that, unlike typical RL, which maximizes the cumulative reward, this optimization performs a \textit{min-over-time} to remember if the trajectory ever entered the failure set. Thus, $\monitor(\latent) < 0$ indicates that the robot is doomed to enter the failure set if starting from $\latent$, whereas $\monitor(\latent) \geq 0$ implies there exists a safety-preserving action (e.g., provided by $\fallback$) that can prevent future failure.

At runtime, this learned safety filter can be deployed to safeguard any arbitrary task policy $\policyTask$ with respect to a \textit{fixed} safety constraint. Given the current latent representation $\latent$ (embedded from the current observations) and the action proposed by the task policy, the safety filter evaluates the sign of $\monitor$ of the next state $\latent' = \dynz(\latent, \policyTask)$ as a safety monitor. Based on this evaluation, the filter either allows the action from $\policyTask$ to proceed or overrides it with the fallback policy:
$\action^\text{exec} = \mathds{1}_{\{\monitor(\latent) > 0\}} \cdot \policyTask + \mathds{1}_{\{\monitor(\latent) \le 0\}} \cdot \fallback(\latent)$.

\section{Constraint-Parameterized \\ Latent Safety Filter}\label{sec:constraint_parameterized_latent_safety_filter}

In this section, we generalize the \textit{fixed} latent safety filter by \textit{parameterizing} it on a constraint representation, $\latentConstraint$, obtained by encoding images of constraints that an end-user cares to prevent, yielding $\monitor(\latent; \latentConstraint)$ and $\fallback(\latent; \latentConstraint)$. We describe three key ingredients--a latent similarity measure, calibration, and training--that enable constraint parameterization of latent safety filters, thereby making them adaptable at test time. The overall framework is described in Fig.~\ref{fig:framework}

\para{Learning a Similarity Measure Over Failures}
Recall how the fixed latent safety filter requires a classifier for a specific constraint. To generalize beyond a fixed classifier, we propose
utilizing a dense similarity measure  $\ellz(\latent; \latentConstraint)$ in the latent space, which represents how close $\latent$ is to the specified representation $\latentConstraint$, thereby enabling any constraint embedding to be used as a possible safety constraint. 
For example, a natural choice could be cosine similarity between any world model state, $\latent$,  and the constraint embedding, $\latentConstraint$. 
However, as we find in Sec.~\ref{sec:exp-dubins-adapt} and Sec.~\ref{sec:exp-runtime-failure}, using the raw embeddings from the world model to compute this similarity measure is often not sufficiently informative nor aligned with an end-user's understanding of similarity. 
To address this, we train a projector $\semanticEncoder: \latentSpace \rightarrow \latentSemanticSpace$ on top of the latent space, which maps the world model's representation $\latent \in \latentSpace$ into a failure-relevant latent representation $\latentSemantic \in \latentSemanticSpace$:
\begin{align}\label{eq:failure-projector}
\text{Failure Projector: }&\latentSemantic = \semanticEncoder(\latent),\\
\text{Latent Failure Margin: }&\tilde{\ell}_z(\latent; \latentConstraint) := -\similarity(\latentSemantic, \latentSemanticConstraint).
\end{align} 
This ensures that the similarity $\similarity(\latentSemantic, \latentSemanticConstraint)$ provides a better aligned measure of meaningful distances in the latent space, enabling safety specifications to be represented by a dense latent failure margin function.

In general, this projector can be trained in various ways to align failure-relevant features, such as supervised metric learning~\cite{kim2020proxy}, representation learning~\cite{eysenbach2022contrastive}, or alignment~\cite{tian2024what}. In this work, we adopt a simple supervised learning so that similarity in the projected latent space reflects proximity to failure-relevant features\footnote{A thorough comparison or the design of novel training objectives for metric-space similarity measures is beyond the scope of this work.}.

Ultimately, safety constraints depend on an embedding $\latentConstraint$ and are modeled via the $\delta$-sub-level set (left, Fig.~\ref{fig:framework}):
\begin{equation}\label{eq:conditioned-failure-set}
    \failure_{\latentConstraint}^{\delta} := \{\latent : \tilde{\ell}_z(\latent; \latentConstraint) \le \delta \} = \{\latent : -\similarity (\latentSemantic, \latentSemanticConstraint) \le \threshold \},
\end{equation} where $\similarity(\cdot, \cdot)$ is the cosine similarity between two vectors and $\threshold$ is a threshold for how similar a latent state must be to the constraint embedding to be also considered a failure.

\para{Calibrating the Similarity-based Latent Failure Set}\label{sec:calib}
Recall that the latent failure set defined in Eqn.~\eqref{eq:conditioned-failure-set} depends on the threshold $\delta$, which determines its effective size. 

We calibrate the threshold $\delta$ that defines the size of the failure set to align with a user's notion of failure set size, thereby controlling how closely the robot is permitted to approach the constraint representation, $\latentConstraint$ (middle left, Fig.~\ref{fig:framework}).

Specifically, we employ Conformal Prediction (CP), a distribution-free statistical method~\cite{shafer2008tutorial, angelopoulos2023conformal}. Using a held-out calibration dataset that reflects the user’s understanding of failures, we aim to provide a recall guarantee for detecting failures conditioned on a representation~\cite{lin2023generating}. 
The calibration dataset consists of latent pairs with ground-truth labels, $\dataset_\calib := \big\{\big((\latent_j, \latent_j'), \gtclass_j \big)\big\}_{j=1}^{N_\calib},$ where $\gtclass_j \in \{0, 1\}$ indicates whether the pair of representations are similar. We then employ class-conditioned CP~\cite{chakraborty2024enhancing, seo2025uncertainty} to guarantee the recall of failures at a user-specified confidence level $\alphaCalibration \in [0,1]$:
\begin{equation}\label{eq:class-conditioned-guarantee}
\mathbb{P}\left( \tilde{\ell}_z(\latent_\test; \latent_\test') \leq \threshold \mid \gtclass_\test = 1 \right) \geq 1 - \alphaCalibration.
\end{equation} 
This is implemented by only using positive latent pairs from the same class (i.e., $\gtclass_i = 1$) and defining the conformal nonconformity score as $-\similarity(\latentSemantic_j, \latentSemantic_j')$. The threshold $\threshold$ is then chosen as the $(1 - \alphaCalibration)$-quantile of the set $\{-\similarity(\latentSemantic_j, \latentSemantic_j')\}_{i=1}^{N}$, obtained by selecting the $\lceil (1 - \alphaCalibration)(N + 1) \rceil$-th smallest value, where $N$ denotes the number of positive pairs in the calibration dataset. 
This class-conditioned CP guarantees the recall of failure~\cite{chakraborty2024enhancing}, $\mathbb{P}\!\left(\latent_\test \in \failure_{\latent_\test'}^{\delta} \mid \gtclass_\test = 1 \right) \;\geq\; 1 - \alphaCalibration$.

\para{Training Constraint-Parameterized Latent Safety Filter}
Finally, we train a constraint-parameterized latent safety filter entirely within the latent imagination of the world model. During training, we treat \textit{any} observation from the world model dataset as a candidate failure we could see at test-time; we randomly sample observations, encode them into $\latentConstraintTraining$, and solve the constraint-parameterized fixed-point safety Bellman equation conditioned on them (top middle, Fig.~\ref{fig:framework}):
\begin{equation}\label{eqn:discounted-safety-bellman-metric}
\begin{aligned}
    \monitor(\latent_t; \latentConstraintTraining) &= (1-\gamma)\,\tilde{\ell}_z(\latentSemantic, \latentSemanticConstraintTraining) \\ + \gamma \min &\Big\{\,\tilde{\ell}_z(\latentSemantic, \latentSemanticConstraintTraining), \max_{\action_t \in \actionSpace} \mathbb{E}_{\hat{\latent}_{t+1} \sim \dynz(\cdot \mid \latent_t, \action_t)}\monitor\big( \hat{\latent}_{t+1}; \latentConstraintTraining \big) \Big\}, 
    \\ \fallback(\latent; \latentConstraintTraining) &= \arg\max_{\action \in \actionSpace} \mathbb{E}_{\hat{\latent}_{t+1} \sim \dynz(\cdot \mid \latent_t, \action)} \monitor( \hat{\latent}_{t+1} ; \latentConstraintTraining),
\end{aligned}
\end{equation} 
where $\gamma \in [0,1)$ is a time discounting factor similar to \eqref{eqn:discounted-safety-bellman}. 

Intuitively, the safety value function $\monitor(\latent_t; \latentConstraintTraining)$ measures how close the robot, starting from $\latentSemantic_t$, comes to the failure representation $\latentSemanticConstraintTraining$ in $\latentSemanticSpace$ despite its best-effort safety policy $\fallback(\latent_t; \latentConstraintTraining)$ to minimize similarity with that representation. As $\latentConstraintTraining$ is sampled randomly from the training dataset, the safety filter can treat either one of these or a newly interpolated representation as a potential failure at test time. 

\para{Runtime Constraint-Parameterized Safety Filtering}
At runtime, a user provides a constraint image $\obs_c$, which is encoded into a constraint representation $\latentConstraint = \encoder(\obs_c)$ and used to adapt the safety filter to the specified constraint (right, Fig.~\ref{fig:framework}). The safety value function evaluates whether the action proposed by the task policy would inevitably enter the latent-space failure set defined by the calibrated threshold, and intervenes with the safety-preserving policy if necessary: 
\begin{equation}\label{eq:latent-conditioned-safe-control}
\action^\text{exec}_{\latentConstraint} :=
\begin{cases}
    \policyTask, & \text{if } \monitor\big(\dynz(\latent, \policyTask); \latentConstraint\big) > \threshold, \\[6pt]
    \fallback(\latent; \latentConstraint), & \text{otherwise}.
\end{cases}
\end{equation} 
Intuitively, the safety filter ensures that the observations--and the corresponding latent states---are \textit{dissimilar} enough (by a $\delta$ margin) to the conditioned runtime safety constraint.

Note that in \eqref{eq:latent-conditioned-safe-control}, we apply the calibrated threshold directly to the value function $\monitor$, while the guarantee in \eqref{eq:class-conditioned-guarantee} only ensures that the similarity measure $\tilde{\ellz}$ is calibrated. Approximate value function solvers (e.g., RL~\cite{fisac2019bridging}) can induce errors, but directly calibrating the value function requires stronger assumptions about access to ground-truth unsafe set labels~\cite{lin2023generating}. We thus apply the similarity-calibrated threshold $\delta$ directly to the value function at runtime, assuming that approximation errors are marginal. Importantly, this enables calibration to be performed post hoc, without the need to retrain the safety filter for different calibration results.\footnote{In the \href{https://any-safe.github.io/}{Appendix}, we prove that with a perfect value function solver, applying the calibrated threshold $\delta$ directly to the value function yields the same unsafe set as solving a threshold-dependent value function.}

\section{Simulation Results}\label{sec:dubins} 

We first conduct experiments with a low-dimensional, benchmark collision-avoidance navigation task where privileged information about the state, dynamics, safe set, and safety controller is available. We focus on the following questions: (\romannumeral 1) Can \anysafe adapt to diverse test-time safety constraints? (\romannumeral 2) Does calibration of \anysafe correctly align the safety filter with the user’s understanding of failure?

\subsection{Experimental Setup}

\para{Privileged Dynamics: 3D Dubins' Car} Let discrete-time dynamics with privileged state be \(\state = [p^x, p^y, \theta]\) $\state_{t+1} = \state_t + \Delta t\, [v\cos(\theta_t),\, v\sin(\theta_t),\, a_t],$ where the robot's action is continuous angular velocity \( a_t \in \mathcal{A} = [-\action_\text{max}, \action_\text{max}]\) with $\action_\text{max} = 1.25$ rad/s,  while the longitudinal velocity is fixed \(v = \SI{1}{m/s}\). The time discretization is \(\Delta t = \SI{0.05}{s}\).

\para{World Model}
We adopt Dreamer~\cite{hafner2023dreamerv3} with the latent dynamics model of the Recurrent State Space Model (RSSM)~\cite{hafner2019learning} with continuous latents. The world model is trained using an offline dataset of $N_\text{train} = 4,000$ observation–action trajectories collected without failure labels. Each observation is a $3\times128 \times 128$ image of the environment and vehicle (Fig.~\ref{fig:dubins}), while actions are randomly sampled during trajectory generation. Each trajectory terminates after $T = 100$ timesteps or earlier if the ground-truth $x$ or $y$ coordinate leaves the environment bounds of $[\SI{-1.5}{m}, \SI{1.5}{m}]$. 

\para{Failure Projector \& Calibration}
We train the failure projector, implemented as a 2-layer MLP, sampling from the world model training dataset to construct $\{(z_{i1}, z_{i2}, s_i)\}$. The ground-truth similarity score is defined as $s_i = \max[1 - \frac{1}{\sqrt{2}} \{(p^x_{i1} - p^x_{i2})^2 + (p^y_{i1} - p^y_{i2})^2\}, -1.0],$ based on the ground-truth robot positions. The failure projector is then trained with the mean-square error (MSE) loss $\tilde{\mathcal{L}}_i =  \Big( \similarity(\latentSemantic_{i1}, \latentSemantic_{i2}) - s_i \Big)^2.$ The calibration dataset consists of $N_\text{calib} = 3,000$ held-out images labeled with robot positions. Pairs within $\epsilon=\SI{0.5}{m}$ are labeled positive, and used to compute the threshold $\delta$ with $\alpha=0.005$.

\begin{figure}[t]
    \centering \includegraphics[width=1.0\linewidth]{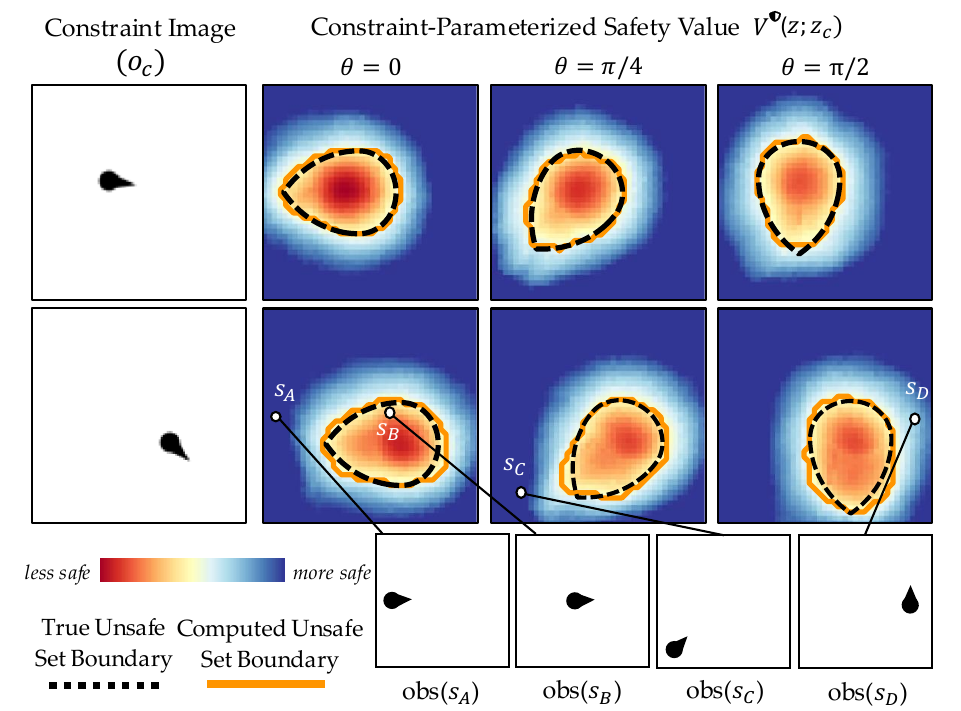}
    \caption{\textbf{Dubins' Car Qualitative Results.} We visualize the safety value function parameterized on two different constraints, shown at heading slices $\theta \in \{0, \pi/4, \pi/2\}$.}
    \label{fig:dubins}
    \vspace{-0.05in}
\end{figure}

\begin{table}[t]
    \Large
    \centering
    \resizebox{1.0\linewidth}{!}{
    \renewcommand{\arraystretch}{1.3}{
    \begin{tabular}{l|ccc|ccccccc}
        \toprule
        \textbf{Method} & 
        $\dynz$ & $\ellz$ & $\tilde{\ell}_z$ &
        \textbf{FPR$\downarrow$} & \textbf{Rec.$\uparrow$} & \textbf{Pre.$\uparrow$} & \textbf{\boldmath$F_1\uparrow$} & \textbf{B.Acc.$\uparrow$} & \textbf{Safe Rate$\uparrow$} \\
        \midrule
        \textit{Privileged-Fix} & 
         \xmark & \xmark & \xmark & 
        0.030 & 0.982 & 0.991 & 0.987 & 0.976 & 0.960\\
        \textit{Privileged-Any} & \xmark & \xmark & \xmark &  0.050 & 0.977 & 0.986 & 0.982 & 0.964 & 0.988 \\  \hline 
        \textit{Latent-Fix}& \cmark & \cmark & \xmark & 0.041 & 0.971 & 0.988 & 0.980 &	0.965 & 0.908\\
        \textit{Latent-Fix-Cont}&  \cmark & \xmark & \cmark &  0.080 & 0.972 & 0.977 & 0.975 & 0.946 & 0.904\\
        \textit{Anysafe (w.o. Proj)} & \cmark & \xmark & \xmark & 0.480 & 0.991 & 0.881 & 0.933 & 0.755 & 0.836\\
        \textit{\ours} & \cmark & \xmark & \cmark & 0.082 & 0.966 & 0.977 & 0.971 & 0.942 & 0.924\\
        \bottomrule
    \end{tabular}
    }
    }
    \caption{\textbf{Comparison of Safety Filter in Dubins' Car.} \ours accurately adapts to different safety constraints while maintaining safety performance comparable to a filter trained for a single safety constraint.}
    \label{tab:dubins_results}
    \vspace{-0.25in}
\end{table}

\para{Latent Safety Filter Setup}
We use DDPG~\cite{lillicrap2015continuous} as our solver for computing the latent safety filter. We randomly sample initial robot observations and the failure observations from the world model training data and encode both into the latent state. For runtime filtering \eqref{eq:latent-conditioned-safe-control}, we add a small margin ($0.1$) to the calibrated threshold.

\para{Evaluation \& Metrics} In this low-dimensional example, we have access to the ground-truth dynamics and can compute a high-quality ``ground-truth'' safety value function using grid-based methods~\cite{mitchell2007toolbox}. This enables us to directly evaluate the safety monitor $\monitor$’s classification accuracy across all three state dimensions. We evaluate the quality of the safety value function conditioned on $50$ different constraint images. We measure the classification accuracy across all three state dimensions. To assess the performance of the fallback policy $\fallback$, we roll out the learned policies from $250$ ground-truth safe initial states and measure the safety rate by checking whether the safety policy ensures the robot never enters the ground-truth failure set.

\subsection{Can \anysafe Adapt to Diverse Safety Constraints?}\label{sec:exp-dubins-adapt}

We first study whether \anysafe can be parameterized on diverse constraints without sacrificing safety performance.

\para{Baselines}
We compare \anysafe against the following baselines. \textit{Privileged} baselines train safety filters using RL with privileged ground-truth states~\cite{fisac2019bridging}, with a single circular failure set of radius $\epsilon = 0.5$ centered at $(\SI{0}, \SI{0})$ (\textit{Privileged-Fix}), or by parameterizing with the position of the constraint in a privileged state-space (\textit{Privileged-Any}). 
\textit{Latent} baselines use the world model, raw image observations as input, and center the failure set in the middle of the plane.
\textit{Latent-Fix}~\cite{nakamura2025generalizing} employs a single failure classifier while \textit{Latent-Fix-Cont} uses a continuous similarity signal $\tilde{\ell}_z$ to the representation of an image with the vehicle at the center. Lastly, \textit{Anysafe (w.o. Proj)} solves \eqref{eqn:discounted-safety-bellman-metric} using raw latent similarities without the failure projector. For baselines restricted to a single safety constraint (the \textit{*-Fix} variants), we report performance only against the ground-truth center failure region.
  
\para{Results} Qualitative results in Fig.~\ref{fig:dubins} show that \ours adapts to diverse failure sets when conditioned on different constraint representations, constructing accurate constraint-parameterized unsafe sets. Quantitative results in Table~\ref{tab:dubins_results} first confirm that latent safety filters achieve performance comparable to safety filters trained with privileged states, consistent with \cite{nakamura2025generalizing}. 
In both privileged-state (\textit{Privileged-Any}) and latent-state (\ours) settings, conditioning on failure states during training enables the safety filter to generalize to arbitrary failure sets. Furthermore, \ours achieves safety monitor performance and safe rates comparable to those of a filter trained for a single fixed failure set ({\textit{Latent-Fix}}), while maintaining the ability to adapt to arbitrary safety constraints. Moreover, \anysafe without the failure projector (\textit{Anysafe (w.o. Proj)}) performs significantly worse with low accuracy and safe rate. This confirms our hypothesis that the raw latent space of the world model is not well aligned for detecting failures and that the failure projector is necessary for effective parameterization of the latent safety filter.

\begin{table}[t]
    \centering
    \resizebox{1.0\linewidth}{!}{
    \renewcommand{\arraystretch}{1.2}{
    \begin{tabular}{l|cccccc}
        \toprule
         \textbf{Method} & \textbf{FPR$\downarrow$} & \textbf{Recall $\uparrow$} & \textbf{Pre.$\uparrow$} & \textbf{\boldmath$F_1\uparrow$} & \textbf{B.Acc.$\uparrow$} & \textbf{Safe Rate$\uparrow$}\\
        \midrule
        $\latentSpace \times \latentSpace$ & 0.082 & 0.966 & 0.977 & 0.971 & 0.942 & 0.924\\
        $\latentSpace \times \mathcal{P}$ & 0.221 & 0.912 & 0.946 & 0.929 & 0.845 & 0.904\\
        $\latentSpace \times \latentSemanticSpace$  & 0.064	& 0.933 & 0.981 & 0.957 & 0.935 & 0.828\\
        $\latentSemanticSpace \times \latentSemanticSpace$ & 0.113 & 0.910 & 0.967 & 0.938 & 0.899 & 0.504\\
        \bottomrule
    \end{tabular}
    }}
    \caption{\textbf{Ablation: Parameterization Strategies.} The constraint-parameterized latent safety filter shows the best performance when parameterized with latent representations, with constraints randomly sampled from the training dataset.}
    \label{tab:dubins_ablations}
    \vspace{-0.25in}
\end{table}

\paranopunc{\textit{Ablation}: How should the latent safety filter parameterize the safety constraint?} We hypothesize that sampling constraint representations from the world model training data enables effective parameterization of the latent safety filter, by exposing it to possible test-time constraints and allowing it to learn to interpolate between them. To test this, we ablate the conditioning strategy for both training \eqref{eqn:discounted-safety-bellman-metric} and runtime filtering \eqref{eq:latent-conditioned-safe-control}. We compare four variants: (\romannumeral 1) $(\latent;\,\latentConstraint) \in \latentSpace \times \latentSpace$, which conditions directly on latent states and constraint representations; (\romannumeral 2) $(\latent;\,\prototype) \in \latentSpace \times \mathcal{P}$, which samples only from a small set of ``prototypes ''of the latent representations $\mathcal{P} \subset \latentSpace$ during training and uses the nearest prototype to the embedding of a user-specified constraint image at runtime. Prototypes are computed via K-means clustering with 9 cluster centers. (\romannumeral 3) $(\latent;\,\latentSemanticConstraint) \in \latentSpace \times \latentSemanticSpace$, which conditions the safety filter on projected latent representations; and (\romannumeral 4) $(\latentSemantic;\,\latentSemanticConstraint) \in \latentSemanticSpace \times \latentSemanticSpace$, which conditions both the current state and the projected latent failure representation.

Table~\ref{tab:dubins_ablations} shows that the latent safety filter performs best when conditioning on the original latent states. Conditioning on diverse failure representations yields high-quality value functions without sacrificing performance, whereas conditioning on a restricted subset $\mathcal{P} \subset \latentSpace$ fails to generalize across diverse failure sets. Notably, conditioning on projected latent representations underperforms, as the projection omits critical state information (e.g., angles) required to accurately recover unsafe sets, preserving only features needed for similarity measures (e.g., positions).

\begin{figure}[t]
    \centering \includegraphics[width=1.0\linewidth]{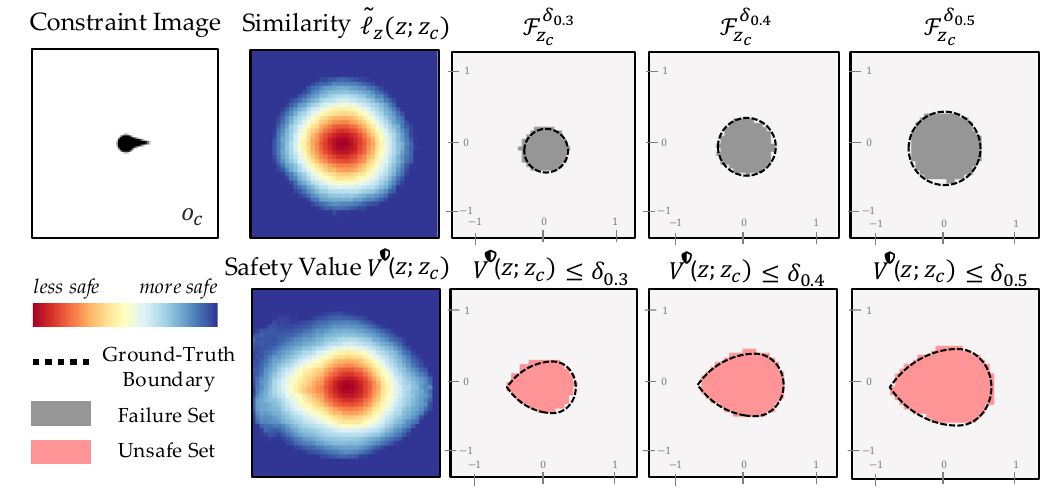}
    \caption{\textbf{Qualitative: Thresholds Calibrated with Different Datasets.} We visualize the failure and unsafe sets constructed with thresholds calibrated from different datasets. \anysafe can adjust the size of the failure sets through calibration.}
    \label{fig:dubins_calib}
    \vspace{-0.1in}
\end{figure}

\begin{table}[t]
    \Large
    \centering
    \renewcommand{\arraystretch}{1.3}
    \resizebox{1.0\linewidth}{!}{
    \begin{tabular}{c|ccccccc}
        \toprule
        \multicolumn{1}{c|}{\textbf{Threshold}} & \textbf{FPR$\downarrow$} & \textbf{Recall $\uparrow$} & \textbf{Pre.$\uparrow$} & \textbf{\boldmath$F_1\uparrow$} & \textbf{B.Acc.$\uparrow$} & \textbf{Safe Rate$\uparrow$} & \textbf{Min. Dist} \\
        \midrule
        \multicolumn{1}{c|}{$\delta=\delta_{0.3}$} & 0.145 & 0.986 & 0.986 & 0.982 & 0.917 & 0.908 & 0.377\\
        \multicolumn{1}{c|}{$\delta=\delta_{0.4}$} & 0.111 & 0.974 & 0.981 & 0.977 & 0.932& 0.880 & 0.477\\
        \multicolumn{1}{c|}{$\delta=\delta_{0.5}$} &  0.082 & 0.966 & 0.977 & 0.971 & 0.942 & 0.924 & 0.569\\
        \bottomrule
    \end{tabular}
    }
    \caption{\textbf{Quantitative: Thresholds Calibrated with Different Datasets.} The same safety filter is used with varying thresholds calibrated from different datasets. Minimum distances to the failure set are measured using ground-truth states and averaged over trajectories. 
    }
    \label{tab:dubins_calib}
    \vspace{-0.2in}
\end{table}

\subsection{Does Calibration Adaptively Control the Failure Set Size?} 

Since \anysafe specifies safety constraints based on continuous latent similarities, the conformal calibration defines the effective size of the failure set and, thus, the unsafe set. To evaluate the conformal calibration process, we test whether the calibrated threshold leads to failure and unsafe sets that align with the user’s understanding of the safety constraints.

\para{Setup} We construct three calibration datasets of equal size, but with labels defined under different criteria for detecting failures, $ y_i = \mathds{1}\!\left\{ (p^x_{i1} - p^x_{i2})^2 + (p^y_{i1} - p^y_{i2})^2 < \epsilon \right\}.$ Specifically, we build $D_\calib^{0.3}, D_\calib^{0.4}, D_\calib^{0.5}$ using thresholds $\epsilon \in \{0.3, 0.4, 0.5\}$, and compute respective $\{\delta_{0.3}, \delta_{0.4}, \delta_{0.5}\}$.

\para{Results} Fig.~\ref{fig:dubins_calib} shows that unsafe sets of different sizes can be accurately estimated using different calibration datasets. The value function learned from continuous similarity signals produces a smooth estimate of the worst-case similarity to the safety constraint. By applying calibrated thresholds to this value function, the safety filter can construct different-sized unsafe sets for runtime filtering. Table~\ref{tab:dubins_calib} further confirms that calibration enables accurate construction of unsafe sets of varying sizes, effectively constraining how closely the system can approach a safety constraint. The calibrated thresholds yield a high-quality value function with high accuracy. It also maintains a high safety rate during rollouts and ensures that minimum distances to the failure position remain above the desired value, which is the radius $\epsilon$ plus a small margin (0.1) that we use for runtime filtering.

\section{Hardware Results: Vision-based Sweeping with a Robotic Manipulator}\label{sec:hardware} 

We scale \anysafe to a real-world visual manipulation task using a Franka Research 3 arm equipped with a third-person camera. The robot is tasked with sweeping small objects on a table using a brush, while avoiding a constraint specified at runtime with an image showing objects in the failure region.

\subsection{Experimental Setup}
\para{World Model} Following \cite{nakamura2025generalizing}, we adopt DINO-WM~\cite{zhou2024dino} as the latent dynamics model, and DINOv2~\cite{oquab2023dinov2} as the encoder. We record $3 \times 244 \times 244$ RGB images at $15$ Hz, along with the end-effector poses. Actions control only the end-effector’s $x-y$ positions and yaw angle. For training the world model, we collect 1,300 trajectories: 300 from teleoperation and 1,000 from random Gaussian-sampled actions.

\begin{figure*}[t!]
    \centering \includegraphics[width=1.0\textwidth]{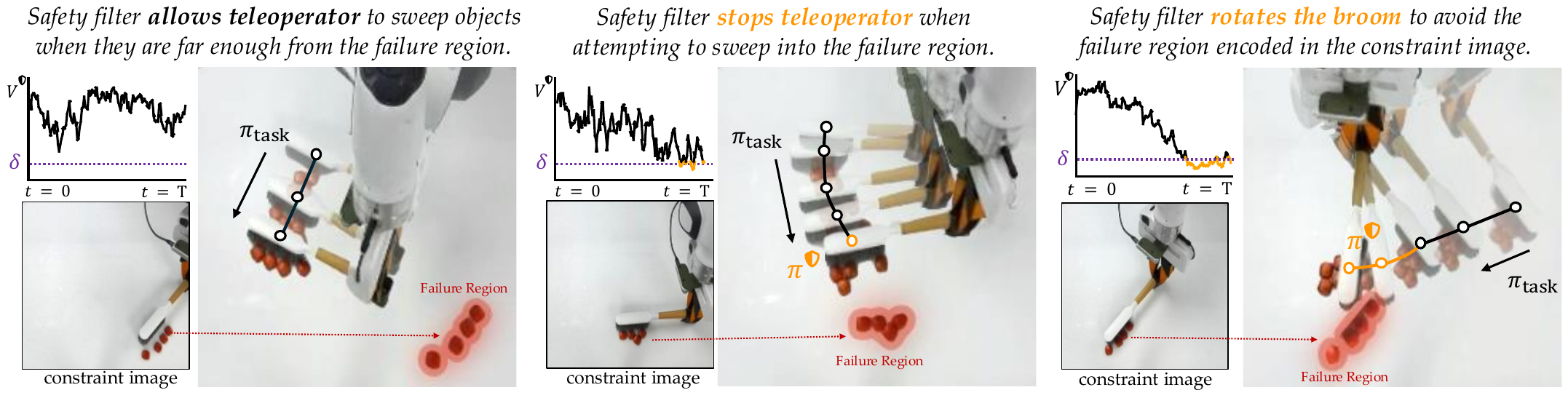}
    \caption{\textbf{Qualitative Results: Filtering $\policyTask$ on Hardware.} \anysafe adapts the safety filter based on the different constraint images, intervening only when the chocolates come close to the runtime failure regions indicated in the constraint images.}
    \label{fig:chocolates}
    \vspace{-0.2in}
\end{figure*}

\para{Failure Projector \& Calibration}
Since the safety specification for this task depends on the positions of objects on the table, labels are generated from similarity scores computed using the pixel coordinates of the objects’ centers of mass. We train a failure projector with $300$ labeled trajectories, sampling latent pairs of images and optimizing an MSE loss as in Sec.~\ref{sec:dubins}. For calibration, we construct a dataset of $N_\text{calib} = 4,000$ labeled images, from which pairs are sampled, and the threshold is then calibrated with $\alpha = 0.1$.

\para{Latent Safety Filter Setup}
We use DDPG~\cite{lillicrap2015continuous} as the solver for the latent safety filter, sampling initial states and constraints from the world model training data.

\begin{figure}[t!]
    \centering \includegraphics[width=1.0\linewidth]{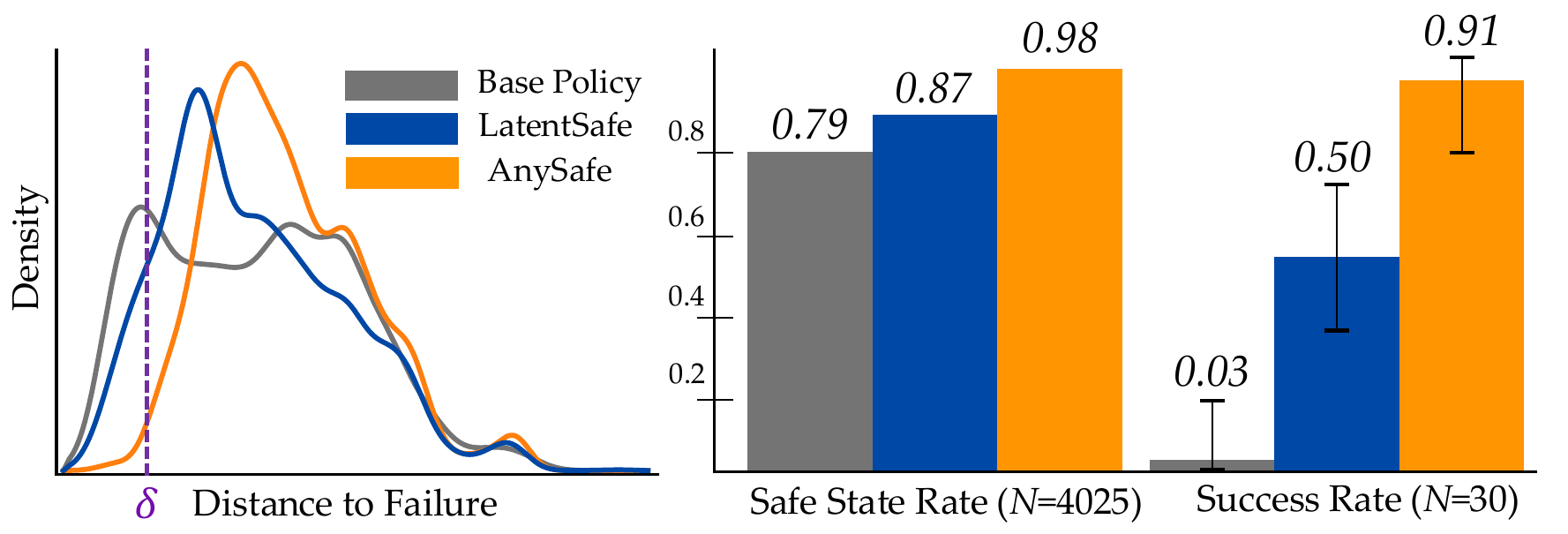}
    \caption{\textbf{Quantitative Results.} \textit{Left}: Empirical distribution of distances between the current object positions and those in the failure images. \textit{Right}: Portion of states outside the failure set and success rate. \ours keeps distances from the runtime failure region above the threshold.
    }
    \label{fig:chocolates_graph}
    \vspace{-0.2in}
\end{figure}

\subsection{Can \anysafe Safeguard a Runtime-Specified Constraint?} 
\label{sec:exp-runtime-failure}

\para{Setup} The task policy $\policyTask$ is human teleoperation, where the teleoperator controls the end-effector’s pose while being filtered by \anysafe. At runtime, the failure representation $\latentConstraint$ is specified by an image depicting an undesirable outcome, such as placing objects within the failure region.

\para{Results: Qualitative} As illustrated in Fig.~\ref{fig:chocolates}, the teleoperator can freely sweep the objects when they are sufficiently far from the safety constraint specified at runtime. When the objects approach the region indicated by the constraint image, however, the safety filter detects the proximity and prevents entry into that region. This safety behavior is conditioned on the constraint representation, showing adaptive behavior for different safety constraints (middle and right of Fig.~\ref{fig:chocolates}). 

\para{Results: Quantitative} We record $30$ trajectories with different safety constraints, each sweeping the objects close to the failure region specified by the constraint image. We then replay these action sequences as $\policyTask$ from the same initial states, with the safety filter adapted to each safety constraint. We compare against \latentsafe~\cite{nakamura2025generalizing}, which trains three separate safety filters for different safety constraints using binary classifiers, each corresponding to one of three distinct regions. At test time, the constraint image is assigned to one of these regions, and the corresponding safety filter is applied to the base policy. Fig.~\ref{fig:chocolates_graph} shows that \ours achieves higher success rates and consistently keeps the distances to the failure regions below the thresholds ($<2\%$). In contrast, \latentsafe fails to satisfy arbitrary safety constraints, highlighting its limited adaptability.

\begin{figure}
    \centering
    \includegraphics[width=1.0\linewidth]{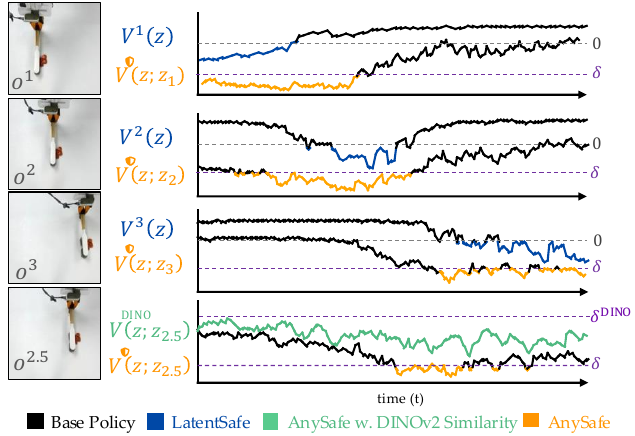}
    \caption{\textbf{Parameterized~(\ours) vs Fixed~(\latentsafe) Safety Filter.} \textit{Left}: Images of three constraints ($o^1, o^2, o^3$) corresponding to each fixed filter. \anysafe adapts to a new constraint, $o^{2.5}$, in between 2 and 3 at test time. \textit{Right}: Safety value functions of the fixed and the parameterized filter. 
    }
    \label{fig:arbitrary_constraint}
    \vspace{-0.25in}
\end{figure}

\subsection{\anysafe's  Generalization Beyond Fixed Safety Filters}\label{exp:generalization}
We next ask whether a single \ours can generalize across diverse constraints better than multiple fixed filters. We compare \ours to three fixed filters (\latentsafe), each specialized on a different failure region with its own failure classifier. We measure the safety values of each method during a replayed open-loop trajectory that sweeps objects through the three failure regions.

\para{Results} Fig.~\ref{fig:arbitrary_constraint} shows that when \ours is adapted with a constraint image corresponding to one of these failure regions ($o^{1}, o^{2}, o^{3}$), it produces a safety value function similar to that of \latentsafe (top three graphs of Fig.~\ref{fig:arbitrary_constraint}). By contrast, when \ours is adapted with a constraint image located between the failure regions ($o^{2.5}$), it identifies the new unsafe set (bottom of Fig.~\ref{fig:arbitrary_constraint}), whereas none of the safety value functions \latentsafe detect it as unsafe. 
Moreover, the \anysafe computed directly from raw DINOv2 representations without the failure projector \eqref{eq:failure-projector} yields ineffective value function estimates (bottom graph of Fig.~\ref{fig:arbitrary_constraint}), as raw DINO similarities are not aligned with safety similarity semantics.

\subsection{Can \anysafe Adapt its Level of Conservativeness?} 

We show that the calibration process allows control over the level of conservativeness by selecting different values of $\alphaCalibration$ in \eqref{eq:class-conditioned-guarantee}, enabling adaptive conservativeness of filtering.

\para{Setup} Using the same calibration dataset, we perform calibration with $\alphaCalibration \in \{0.01, 0.1\}$ to get $\delta_{0.1}$ and $\delta_{0.01}$. For each calibrated threshold, we replay $15$ action sequences as $\policyTask$ that attempt to sweep the objects in the same direction from different starting points, while \ours filters the $\policyTask$.

\begin{figure}[t]
    \centering \includegraphics[width=1.0\linewidth]{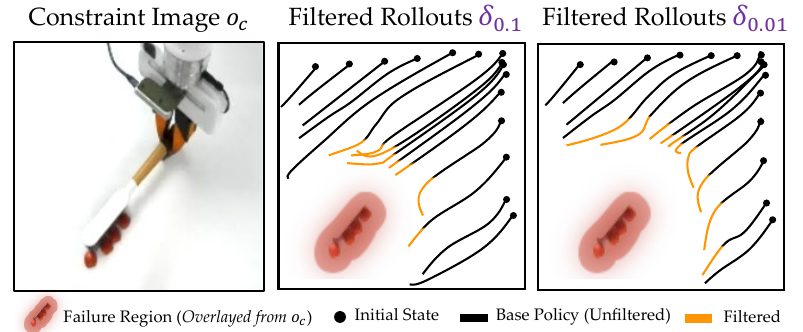} \caption{\textbf{Safety Filtering with Different Conservativeness.} Smaller $\alpha$ in calibration leads to more conservative filtering.}
    \label{fig:sweep_calib}
    \vspace{-0.25in}
\end{figure}

\para{Results} Fig.~\ref{fig:sweep_calib} shows the filtered trajectories with different thresholds. It shows that \ours can enforce different levels of conservativeness depending on $\alphaCalibration$ during the calibration. While consistently preventing the system from entering the failure region specified by the constraint images, a smaller $\alpha$ yields larger margins from the failure region by activating the safety filter more proactively with a calibrated threshold.

\section{Conclusion \& Limitations}
\vspace{-0.02in}
In this work, we propose \anysafe, a constraint-parameterized latent safety filter that adapts to runtime safety constraints. To enable constraint parameterization in latent space, we condition safety filters on constraint representations encoded from images, using a latent-space similarity measure and conformal calibration to align similarities with the user’s notion of failure. Through experiments, we show that \anysafe adapts to arbitrary safety constraints specified by a failure image, while also allowing control over the conservativeness of the filtering through conformal calibration. 

\vspace{-0.05in}
\para{Limitations} Our work assumes adaptation only to images seen during world model training, leaving open the challenge of generalizing to unseen constraints. It also relies on dense labels for learning failure-related similarity measures, motivating future work on data-efficient metric learning to better align similarity with a user’s notion of failure.
\vspace{-0.02in}
\addtolength{\textheight}{0cm}

\bibliographystyle{IEEEtran}
\bibliography{mybib}

\ifanonymous
\else
  \section{Appendix}

\subsection{Implementation Details}

\para{Dubins Car} For the Dubins Car experiments, we use a Recurrent State-Space Model (RSSM)~\cite{hafner2019learning} as the world model. RSSM decomposes the latent state into deterministic and stochastic components, $\latent_t := [h_t | x_t]$, where the stochastic latent is modeled as a distribution and optimized via the KL divergence between its prior and posterior. We build on the open-source implementation of DreamerV3\footnote{\url{https://github.com/NM512/dreamerv3-torch}}. We use a continuous stochastic latent space modeled as a 32-dimensional Gaussian. The action space is normalized to [-1, 1]. The hyperparameters for the Dubins Car experiments are provided in Table~\ref{tab:dubins_hyperparams}.
\begin{table}[ht]
    \centering
    \small{
    \renewcommand{\arraystretch}{1.1}{
    \resizebox{0.8\linewidth}{!}{
    \begin{tabular}{lc}
        \toprule
        \textbf{\textsc{Hyperparameter}} & \textbf{\textsc{Value}} \\
        \midrule
        \textsc{Image Dimension} & [128, 128, 3] \\
        \textsc{Action Dimension} & 1 \\
        \textsc{Stochastic Latent} & Gaussian \\
        \textsc{Latent Dim (Deterministic)} & 512 \\
        \textsc{Latent Dim (Stochastic)} & 32 \\
        \textsc{Latent Dim (Failure Projector)} & 512 \\
        \textsc{Activation Function} & SiLU \\
        \textsc{Encoder CNN Depth} & 32 \\
        \textsc{Encoder MLP Layers} & 5 \\
        \textsc{Failure Projector Layers} & 2 \\
        \textsc{Batch Size} & 16 \\
        \textsc{Batch Length} & 64 \\
        \textsc{Optimizer} & Adam \\
        \textsc{Learning Rate} & 1e-4 \\
        \textsc{Iterations} & 10000 \\
        \bottomrule
        \end{tabular}
        }
    }}
    \caption{Dreamer Hyperparameters}
    \label{tab:dubins_hyperparams}
    \vspace{-0.2in}
\end{table}

\para{Vision-Based Sweeping} For the real-world hardware task, we use DINO-WM~\cite{zhou2024dino}, a transformer-based world model that represents latent states using the patch tokens of DINOv2~\cite {oquab2023dinov2}. The DINOv2 encoder is kept frozen, and only the parameters of a vision transformer are trained. The latent $\hat{\latent}_{t+1}$ consisting of dense patch tokens is deterministically predicted by conditioning on a sequence of past normalized actions $\action_{t-H:t}$ and latent tokens $\latent_{t-H:t}$ from the previous $H$ timesteps. The model is trained with teacher forcing to ensure temporal consistency by regressing the latent patches, using the mean squared error (MSE) loss: $
\mathcal{L}_{\text{DINO}} = \left\| \encoder(\obs_{t+1}) - \dynamics_z (\latent_{t-H:t}, \action_{t-H:t}) \right\|^2.
$ The hyperparameters for DINO-WM are provided in Table~\ref{tab:dino_hyperparams}. To measure the similarity of DINOv2 features in Sec.~\ref{exp:generalization}, we compute the norm of the patch tokens to obtain a global feature $\latent \in \mathbb{R}^{384}$ from the dense patches $\latent \in \mathbb{R}^{N_\text{patches} \times 384}$.

\begin{table}[ht]
    \centering
    \small{
    \renewcommand{\arraystretch}{1.1}{
    \resizebox{0.8\linewidth}{!}{
    \begin{tabular}{lc}
        \toprule
        \textbf{\textsc{Hyperparameter}} & \textbf{\textsc{Value}} \\
        \midrule
        \textsc{Image Dimension} & [224, 224, 3] \\
        \textsc{Action Dimension} & 3 \\
         \textsc{DINOv2 Patch Size} & (16 $\times$ 16, 384)\\
         \textsc{ViT depth} & 6 \\
         \textsc{ViT attention heads} & 16 \\
         \textsc{ViT MLP dim} & 2048 \\
        \textsc{Latent Dim (Failure Projector)} & 512 \\
        \textsc{Activation Function} & SiLU \\
        \textsc{Failure Projector Layers} & 2 \\
        \textsc{Batch Size} & 16 \\
        \textsc{Batch Length} & 4 \\
        \textsc{Optimizer} & Adam \\
        \textsc{Learning Rate} & 5e-5 \\
        \textsc{Iterations} & 100000 \\
        \bottomrule
        \end{tabular}
        }
    }}
    \caption{DINO-WM Hyperparameters}
    \label{tab:dino_hyperparams}
    \vspace{-0.2in}
\end{table}

\para{HJ Reachability Analysis} To solve the latent fixed-point safety Bellman equation, we adopt DDPG~\cite{lillicrap2015continuous} within an off-policy, model-based reinforcement learning framework, using the implementation from~\cite{li2025certifiable}\footnote{\url{https://github.com/jamesjingqili/Lipschitz_Continuous_Reachability_Learning}}. We model the safety value function as a latent-action value function conditioned on the constraint representation $Q(\latent, \action; \latentConstraint)$. The safety policy is parameterized by an actor-network \(\action = \fallback(\cdot \mid \latent, \latentConstraint) \in [-1, 1]^{d_{\text{action}}}\), and the safety value is evaluated by $\monitor(\latent) = \max_aQ(\latent, \action; \latentConstraint) = Q(\latent,\fallback(\latent, \latentConstraint))$.

Each trajectory starts from random initial states with randomly sampled constraint representations. The replay buffer $\mathcal{B}$ then stores transitions of the form $(\latent, \latentConstraint, \action, \tilde{\latentfailuremargin}, \latent')$. The safety filter is optimized using the following objectives:
\begin{align}
       \mathcal{L}_\text{critic} &:= \mathbb{E}_{ \mathcal{B}} \left[ \left(Q(\latent, \action; \latentConstraint) - y\right)^2 \right]
       \\
        y &= (1 - \gamma)\, \tilde{\latentfailuremargin} + \gamma \min\{\tilde{\latentfailuremargin}, \max_{\action'} Q(\latent', \action'; \latentConstraint)\}.  \\
    \mathcal{L}_{\text{actor}} &:= \mathbb{E}_{\latent \sim \mathcal{B}} \left[ -Q(\latent, {\action}; \latentConstraint)\right], \quad {\action} = \fallback(\cdot \mid \latent, \latentConstraint),
\end{align} where $\gamma$ is scheduled from $0.85$ to $0.9999$. The hyperparameters for training DDPG are summarized in Table.~\ref{tab:DDPG_hyperparams}.

\begin{table}[h]
    \centering
    \resizebox{0.8\linewidth}{!}{
    \renewcommand{\arraystretch}{1.1}
    \begin{tabular}{lc}
        \toprule
        \textsc{\textbf{Hyperparameter}} & \textsc{\textbf{Value}} \\
        \midrule
        \textsc{Actor Architecture} & [512, 512, 512, 512] \\
        \textsc{Critic Architecture} & [512, 512, 512, 512] \\
        \textsc{Normalization} & LayerNorm \\
        \textsc{Activation} & ReLU \\
        \textsc{Discount Factor} $\gamma$ & 0.9999 \\
        \textsc{Learning Rate (Critic)} & 1e-3 \\
        \textsc{Learning Rate (Actor)} & 1e-4 \\
        \textsc{Optimizer} & AdamW \\
        \textsc{Number of Iterations} & 640000 \\
        \textsc{Replay Buffer Size} & 1000000 \\
        \textsc{Batch Size} & 512 \\
        \textsc{Max Imagination Steps} & 30 \\
        \bottomrule
    \end{tabular}
    }
    \caption{DDPG hyperparameters.}
    \label{tab:DDPG_hyperparams}
    \vspace{-0.2in}
\end{table}

\subsection{Failure Projector}

In Sec.~\ref{sec:constraint_parameterized_latent_safety_filter} and Eq.~\eqref{eq:failure-projector}, we define a failure projector that maps the raw latent representation of the world model into a metric space where similarities are better aligned with the user’s notion of failure. Since \anysafe constructs the failure set based on a latent-space similarity metric, the projected latent space retains only failure-relevant features. In Sec.~\ref{exp:generalization}, we qualitatively demonstrate that the raw DINOv2 latent space is insufficient to define a failure set, yielding a noisy and uninformative similarity metric that cannot capture fine-grained position-based differences, which in turn degrades the quality of the value function.

\para{Architecture} We implement the failure projector as a 2-layer MLP, with the architecture summarized in Table~\ref{tab:failure_architecture}. For RSSM, the input is the latent vector formed by concatenating the deterministic and stochastic components, $\latent_t = [h_t \, | \, x_t]$, resulting in an input dimension of $d_\text{in} = 512 + 32 = 544$. For DINO-WM, we compute the norm of the patch tokens from the dense latent features $\latent \in \mathbb{R}^{N_\text{patches} \times 384}$ and concatenate the proprioceptive state, yielding an input dimension of $d_\text{in} = 384 + 3 = 387$.

\begin{table}[ht]
\centering
\small
\resizebox{0.8\linewidth}{!}{
\renewcommand{\arraystretch}{1.15}{
    \begin{tabular}{lccc}
    \toprule
    \textbf{\textsc{Layer}} & \textbf{\textsc{Input Dim}} & \textbf{\textsc{Output Dim}} & \textbf{\textsc{Normalization}} \\
    \midrule
    \texttt{Linear} & \(d_\text{in}\) & \(d_\text{in}\) & \texttt{LayerNorm} \\
    \texttt{Linear} & \(d_\text{in}\) & \(32\) & \texttt{LayerNorm} \\
    \bottomrule
    \end{tabular}
}}
\caption{Failure Projector Architecture}
\label{tab:failure_architecture}
\vspace{-0.2in}
\end{table}

\para{Ablation} We provide a qualitative evaluation of the raw world model latent space of Dreamer in the Dubins Car. Fig.~\ref{fig:dubins_appendix} shows both the feature similarities computed using unprojected world model latents and the corresponding safety value function learned with these latents. Compared to the results learned from projected features shown in Fig.~\ref{fig:dubins} and Fig.~\ref{fig:dubins_calib}, the unprojected similarities fail to represent position-based, failure-relevant distinctions, resulting in poorly learned value functions.

\begin{figure}[t]
    \centering \includegraphics[width=1.0\linewidth]{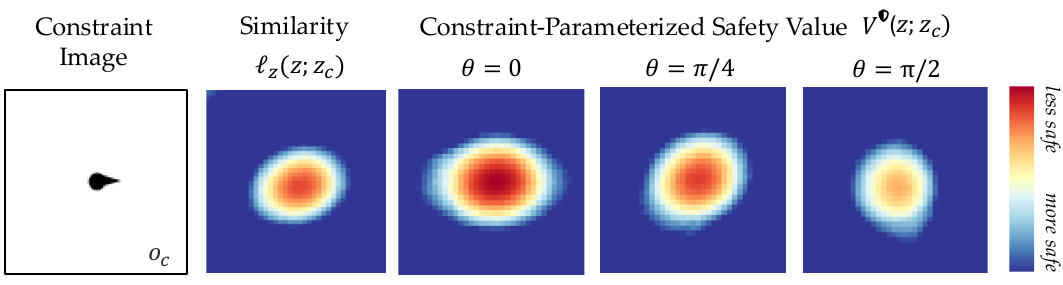}
    \caption{\textbf{Dubins' Car Qualitative Result with Raw Unprojected Features.} We visualize the latent similarity and safety value function at heading slices $\theta \in \{0, \pi/4, \pi/2\}$ using the raw Dreamer features without applying the failure projector.}
    \label{fig:dubins_appendix}
    \vspace{-0.2in}
\end{figure}

\subsection{Calibration of Latent Similarity}

In Sec.~\ref{sec:calib}, we compute a threshold for defining a latent failure set and then apply this threshold at runtime for safety filtering in \eqref{eq:latent-conditioned-safe-control}, operating on the safety value function. Formally, this calibration procedure guarantees only that the similarity measure—and the corresponding failure set—are calibrated. Approximate value function solvers (e.g, RL) can still induce errors in the downstream safety filter, and calibrating the value function directly requires stronger assumptions about access to the ground-truth unsafe set labels~\cite{lin2023generating}. Nevertheless, we prove in Theorem~\ref{theorem:calibration} that under a perfect value function solver, the calibrated threshold $\delta$ can be applied directly to the value function learned from the similarity measure, yielding an unsafe set defined as the sub-threshold level set of the value function.

Specifically, we show that learning a threshold-dependent safety value function $\monitorDelta$ with a safety margin $\ellz^\delta := \ellz - \delta$, is equivalent to learning the value function $\monitor$ from the raw latent similarity $\ellz$ and applying the threshold $\delta$ post hoc. This equivalence implies that calibration can be performed after computing the value function using the latent similarities without adjusting them with thresholds, whereas threshold-dependent training would require recomputing the value function whenever the threshold changes.

\begin{theorem}\label{theorem:calibration}
Let $\ellz:\mathcal{Z}\to\mathbb{R}$ be a safety margin function. Consider the discrete-time safety Bellman backup operator:
\begin{align}
(T\valfunc)(\latent_t)
  &= (1-\gamma)\,\ell(\latent_t) \\
  &\quad + \gamma \min \Big\{ \ell(\latent_t),\;
       \max_{\action_t \in \mathcal{A}}
       \mathbb{E}_{\hat{\latent}_{t+1} \sim \dynz(\cdot \mid \latent_t, \action_t)}
       \big[ V(\hat{\latent}_{t+1}) \big] \Big\}. \nonumber
\end{align}  
Let $\monitor$ denote its unique fixed point. For a constant $\delta\in\mathbb{R}$, define $\ellz^\delta:=\ellz-\delta$ and let $T_\delta$ be the safety Bellman backup operator $T$ where the margin function $\ellz$ is replaced by $\ellz^\delta$:
\begin{align}
(T_\delta\valfunc)&(\latent_t)
  = (1-\gamma)\, (\ell(\latent_t) - \delta) \\
  &\quad + \gamma \min \Big\{ \ell(\latent_t) - \delta,\;
       \max_{\action_t \in \mathcal{A}}
       \mathbb{E}_{\hat{\latent}_{t+1} \sim \dynz(\cdot \mid \latent_t, \action_t)}
       \big[ V(\hat{\latent}_{t+1}) \big] \Big\}. \nonumber
\end{align}
Let the fixed point of this Bellman backup be $\monitorDelta$. Then:
\begin{align}
&\monitorDelta \;=\; \monitor - \delta \quad \text{and}\\
&\{\,z:\; \monitorDelta(z) < 0 \,\} \;=\; \{\,z:\; \monitor(z) < \delta \,\}.
\end{align}
\end{theorem}

\begin{proof}
Let $\valfunc : \latentSpace \rightarrow \mathbb{R}$ be any value function. Recall the linearity of expectation and the shift-invariance identities:
\begin{align*}
        \min\{a-\delta,b-\delta\} \;=\; \min\{a,b\}-\delta, \\ \max\{a-\delta,b-\delta\} \;=\; \max\{a,b\}-\delta. 
\end{align*} 
Using $\ellz^\delta=\ellz-\delta$, we compute
\begin{align*}
(T_\delta (\valfunc-\delta))(\latent_t)
&= (1-\gamma)\,(\ellz(\latent_t)-\delta)\\
& + \gamma \min\!\Big\{\ellz(\latent_t)-\delta,\;
       \max_{\action_t\in\mathcal{A}}
       \mathbb{E}\big[\,\valfunc(\hat{\latent}_{t+1})-\delta\,\big]\Big\} \\
&= (1-\gamma)\,\ellz(\latent_t) - (1-\gamma)\delta \\
& + \gamma \Big(\min\!\Big\{\ellz(\latent_t),\;
       \max_{\action_t\in\mathcal{A}}
       \mathbb{E}\big[\valfunc(\hat{\latent}_{t+1})\big]\Big\} - \delta\Big) \\
&= \Big((1-\gamma)\,\ellz(\latent_t)
        \\& + \gamma \min\!\Big\{\ellz(\latent_t),\;
            \max_{\action_t\in\mathcal{A}}
            \mathbb{E}\big[\valfunc(\hat{\latent}_{t+1})\big]\Big\}\Big) - \delta \\
&= (T \valfunc)(\latent_t) - \delta.
\end{align*}
In particular, if $\valfunc := \monitor$ is the fixed point of $T$, then, 
\begin{equation}
    T_\delta(\monitor-\delta) = T\monitor - \delta = \monitor - \delta,
\end{equation} where $\monitor-\delta$ is a fixed point of $T_\delta$. Since the fixed point is unique \cite{fisac2019bridging}, $\monitorDelta=\monitor-\delta$. 

Thus, the corresponding unsafe sets are equivalent:
\begin{align}
    \{\,\latent:\; \monitorDelta(\latent) < 0 \,\} &=\{\,\latent:\; \monitor(\latent)-\delta < 0 \,\} \nonumber \\
    &=\{\,\latent:\; \monitor(\latent) < \delta \,\}.
\end{align}
\end{proof}

\fi

\end{document}